\documentclass[a4paper,twocolumn]{article}

\usepackage[utf8]{inputenc}
\usepackage{microtype}
\usepackage{hyperref}
\usepackage{graphicx}
\usepackage{booktabs} 
\usepackage{authblk}
\usepackage{placeins}

\usepackage{tikz}
\usepackage[export]{adjustbox}

\usepackage{amsthm}
\usepackage{amsfonts}

\usepackage{subfig}
\usepackage{caption}

\title{Variational Auto-Encoders: Not all failures are equal}
\date{~}
\author[1]{Victor Berger}
\author[1]{Michele Sebag}
\affil[1]{TAU, CNRS $-$ INRIA $-$ Univ. Paris-Saclay, France}

\DeclareUnicodeCharacter{2212}{-}

\begin{document}

\maketitle

\newtheorem{theorem}{Theorem}

\begin{abstract}
We claim that a source of severe failures for Variational Auto-Encoders is the choice of the distribution class used for the observation model. 
A first theoretical and experimental contribution of the paper is to establish that even in the large sample limit with arbitrarily powerful neural architectures and latent space, the VAE fails 
if the sharpness of the distribution class does not match the scale of the data.
Our second claim is that  the distribution sharpness must preferably be learned by the VAE (as opposed to, fixed and optimized offline): Autonomously adjusting this sharpness allows the VAE to dynamically control the trade-off between the optimization of the reconstruction loss and the latent compression. A second empirical contribution is to show how the control of this trade-off is instrumental in escaping poor local optima, akin a simulated annealing schedule. 
Both claims are backed upon experiments on artificial data, MNIST and CelebA, showing how sharpness learning addresses the notorious VAE blurriness issue.
\end{abstract}

\section{Introduction}
A huge leap forward to the identification of generative models from data, Variational Auto-Encoders (VAEs) \cite{kingma_semi-supervised_2014,rezende_stochastic_2014} can be analysed as tackling two complementary problems: i) identifying the data manifold (Pb (1)); ii) identifying a distribution on this manifold that actually matches the data distribution (Pb (2)). 
The VAE training criterion tackles both problems through the joint optimization of a data fitting term and a regularization term enforcing the smoothness of the generative model w.r.t. the chosen distribution class (section \ref{sec:Formal}).

The first claim of the paper is that Pb (1) is the most severe bottleneck for the VAE: if the data manifold is not properly identified, the VAE fails to learn any relevant probability distribution. This failure can be due to several reasons, e.g. insufficiently powerful neural architectures
or reconstruction losses \cite{larsen_autoencoding_2016,so_nderby_ladder_2016,huang_introvae_2018}. We posit that a necessary condition for a successful VAE training is an appropriate choice of the distribution class of the observation model, and that an inappropriate choice {\em cannot be compensated for}.
This claim is backed upon theoretical and experimental results on a simple artificial problem,\footnote{ Quoting Ali Rahimi (2017): {\em  Simple experiments, simple theorems are the building blocks that help us understand more complicated systems}.} i) showing that even in the large sample limit with an infinite representation power of the neural architecture, a VAE does fail if the sharpness of the distribution class does not match the data scale; and ii) analytically characterizing the mismatch between the distribution sharpness and the data scale. 

The second claim of the paper is that a most advantageous option is to let the VAE learn the appropriate sharpness of the observation model, defining the L-VAE setting.
Empirical evidence shows that a L-VAE $-$ considering a flexible distribution class $p_\theta$ and learning the best hyper-parameter $\theta^*$ thereof $-$ behaves significantly better than a VAE considering the fixed distribution class $p_{\theta^*}$. The inspection of the training dynamics shows that the autonomous adjustment of the sharpness allows the L-VAE to control the trade-off between the data fitting and the regularization terms.
In contrast, considering a distribution with a fixed though proper sharpness 
makes the data fitting term dominate the optimization in the early learning phase, engaging the VAE in unpromising regions w.r.t. the optimization of the regularization objective.
These observations, backed upon experiments on artificial data, MNIST and CelebA, are interpreted in terms of a simulated annealing process, as analyzed in  \cite{huang_improving_2018} w.r.t. the dimensionality of a neural net. The lesson learned is that the learning process should better operate in a larger space than the one containing the eventual solution, for the sake of a less constrained and more efficient optimization process. 

The paper is organized as follows. Section \ref{sec:Formal} introduces the state of the art. Section \ref{sec:th} presents a theoretical analysis of the impact of the observation model, and describes the L-VAE setting. Section \ref{sec:dyn} reports and discusses empirical evidence of the merits of L-VAE.  The concluding remarks present some perspectives for further research.

\section{Formal Background and Related Work} \label{sec:Formal}
Now a classical approach for generative modelling, Variational Auto-Encoders \cite{rezende_stochastic_2014,kingma_auto-encoding_2014} search for a generative distribution $p_\theta(x)$, expressed w.r.t. some latent variable $z$:
\begin{equation}
p_\theta(x) = \int p_\theta(x|z) p_\theta(z) dz \label{eq:generative-model}
\end{equation}
The encoder module yields the parameters of distribution $q(z|x)$, defined on the latent space and used to sample $z$. 
The decoder module yields the parameters of distribution $p_\theta(x|z)$, defined on the input space and used to generate (ideally) a new sample of the data distribution. 
The VAE loss is composed of a data fitting term (maximizing the data likelihood) and a latent compression term (enforcing that $q(z|x)$ stay close to the latent prior $p_\theta(z)$):
\begin{equation}
    \mathcal{L}_{VAE}(x) =  \mathop{\mathbb{E}}_{z \sim q} - \log p_\theta(x|z) + D_{KL}(q(z|x)\|p_\theta(z)) \label{eq:vae-loss}
\end{equation}
The data fitting term is optimized by taking advantage of the Evidence Lower Bound (ELBO), stating that for any distribution $q(z|x)$:
\begin{equation}
    \log p_\theta(x) \geq \mathbb{E}_{z \sim q(z|x)} \log \frac{p_\theta(x|z) p_\theta(z)}{q(z|x)} \label{eq:ELBO}
\end{equation}
The minimization of Eq. (\ref{eq:vae-loss}) enables the end-to-end training of the VAE (Fig.  \ref{fig:vae}), specifically the encoder $q(z|x)$ (a.k.a. inference model) and decoder $p_\theta(x|z)$ (a.k.a. observation model), both implemented as neural networks. Both encoder and decoder take a value ($x$ or $z$) as input and produce the parameters of a distribution ($q(z|x)$ or $p(x|z)$). The mainstream class of distribution chosen for the inference model is a Normal distribution, the mean and variance of which are computed by the encoder along the known re-parametrization trick, $q(z|x) = {\cal N}(\mu(x), \sigma(x))$ with $\mu(x)$ and $\sigma(x)$ provided by the encoder module.

\begin{figure}
    \centering
    \resizebox{0.9\columnwidth}{!}{
\ifx\du\undefined
  \newlength{\du}
\fi
\setlength{\du}{15\unitlength}
\begin{tikzpicture}
\pgftransformxscale{1.000000}
\pgftransformyscale{-1.000000}
\definecolor{dialinecolor}{rgb}{0.000000, 0.000000, 0.000000}
\pgfsetstrokecolor{dialinecolor}
\definecolor{dialinecolor}{rgb}{1.000000, 1.000000, 1.000000}
\pgfsetfillcolor{dialinecolor}
\pgfsetlinewidth{0.050000\du}
\pgfsetdash{}{0pt}
\pgfsetdash{}{0pt}
\pgfsetroundjoin
{\pgfsetcornersarced{\pgfpoint{1.000000\du}{1.000000\du}}\definecolor{dialinecolor}{rgb}{0.898039, 0.898039, 0.898039}
\pgfsetfillcolor{dialinecolor}
\fill (14.500000\du,-1.500000\du)--(14.500000\du,7.000000\du)--(19.100000\du,7.000000\du)--(19.100000\du,-1.500000\du)--cycle;
}{\pgfsetcornersarced{\pgfpoint{1.000000\du}{1.000000\du}}\definecolor{dialinecolor}{rgb}{0.000000, 0.000000, 0.000000}
\pgfsetstrokecolor{dialinecolor}
\draw (14.500000\du,-1.500000\du)--(14.500000\du,7.000000\du)--(19.100000\du,7.000000\du)--(19.100000\du,-1.500000\du)--cycle;
}\pgfsetlinewidth{0.050000\du}
\pgfsetdash{}{0pt}
\pgfsetdash{}{0pt}
\pgfsetroundjoin
{\pgfsetcornersarced{\pgfpoint{1.000000\du}{1.000000\du}}\definecolor{dialinecolor}{rgb}{0.898039, 0.898039, 0.898039}
\pgfsetfillcolor{dialinecolor}
\fill (26.000000\du,-1.500000\du)--(26.000000\du,7.000000\du)--(30.500000\du,7.000000\du)--(30.500000\du,-1.500000\du)--cycle;
}{\pgfsetcornersarced{\pgfpoint{1.000000\du}{1.000000\du}}\definecolor{dialinecolor}{rgb}{0.000000, 0.000000, 0.000000}
\pgfsetstrokecolor{dialinecolor}
\draw (26.000000\du,-1.500000\du)--(26.000000\du,7.000000\du)--(30.500000\du,7.000000\du)--(30.500000\du,-1.500000\du)--cycle;
}\pgfsetlinewidth{0.100000\du}
\pgfsetdash{}{0pt}
\pgfsetdash{}{0pt}
\pgfsetroundjoin
{\pgfsetcornersarced{\pgfpoint{0.500000\du}{0.500000\du}}\definecolor{dialinecolor}{rgb}{0.000000, 0.000000, 0.000000}
\pgfsetstrokecolor{dialinecolor}
\draw (20.000000\du,-1.000000\du)--(20.000000\du,2.000000\du)--(25.000000\du,2.000000\du)--(25.000000\du,-1.000000\du)--cycle;
}
\definecolor{dialinecolor}{rgb}{0.000000, 0.000000, 0.000000}
\pgfsetstrokecolor{dialinecolor}
\node at (22.500000\du,0.521559\du){Encoder};
\pgfsetlinewidth{0.100000\du}
\pgfsetdash{}{0pt}
\pgfsetdash{}{0pt}
\pgfsetroundjoin
{\pgfsetcornersarced{\pgfpoint{0.500000\du}{0.500000\du}}\definecolor{dialinecolor}{rgb}{0.000000, 0.000000, 0.000000}
\pgfsetstrokecolor{dialinecolor}
\draw (20.000000\du,2.500000\du)--(20.000000\du,5.500000\du)--(25.000000\du,5.500000\du)--(25.000000\du,2.500000\du)--cycle;
}
\definecolor{dialinecolor}{rgb}{0.000000, 0.000000, 0.000000}
\pgfsetstrokecolor{dialinecolor}
\node at (22.500000\du,4.021559\du){Decoder};
\pgfsetlinewidth{0.100000\du}
\pgfsetdash{}{0pt}
\pgfsetdash{}{0pt}
\pgfsetbuttcap
{
\definecolor{dialinecolor}{rgb}{0.000000, 0.000000, 0.000000}
\pgfsetfillcolor{dialinecolor}
\pgfsetarrowsend{stealth}
\definecolor{dialinecolor}{rgb}{0.000000, 0.000000, 0.000000}
\pgfsetstrokecolor{dialinecolor}
\draw (27.000000\du,4.000000\du)--(25.000000\du,4.000000\du);
}
\pgfsetlinewidth{0.100000\du}
\pgfsetdash{}{0pt}
\pgfsetdash{}{0pt}
\pgfsetbuttcap
{
\definecolor{dialinecolor}{rgb}{0.000000, 0.000000, 0.000000}
\pgfsetfillcolor{dialinecolor}
\pgfsetarrowsend{stealth}
\definecolor{dialinecolor}{rgb}{0.000000, 0.000000, 0.000000}
\pgfsetstrokecolor{dialinecolor}
\draw (20.000000\du,4.000000\du)--(18.000000\du,4.000000\du);
}
\pgfsetlinewidth{0.100000\du}
\pgfsetdash{}{0pt}
\pgfsetdash{}{0pt}
\pgfsetbuttcap
{
\definecolor{dialinecolor}{rgb}{0.000000, 0.000000, 0.000000}
\pgfsetfillcolor{dialinecolor}
\pgfsetarrowsend{stealth}
\definecolor{dialinecolor}{rgb}{0.000000, 0.000000, 0.000000}
\pgfsetstrokecolor{dialinecolor}
\draw (18.000000\du,0.500000\du)--(20.000000\du,0.500000\du);
}
\pgfsetlinewidth{0.100000\du}
\pgfsetdash{}{0pt}
\pgfsetdash{}{0pt}
\pgfsetbuttcap
{
\definecolor{dialinecolor}{rgb}{0.000000, 0.000000, 0.000000}
\pgfsetfillcolor{dialinecolor}
\pgfsetarrowsend{stealth}
\definecolor{dialinecolor}{rgb}{0.000000, 0.000000, 0.000000}
\pgfsetstrokecolor{dialinecolor}
\draw (25.000000\du,0.500000\du)--(27.000000\du,0.500000\du);
}
\definecolor{dialinecolor}{rgb}{0.000000, 0.000000, 0.000000}
\pgfsetstrokecolor{dialinecolor}
\node at (28.200000\du,4.021559\du){$z$};
\definecolor{dialinecolor}{rgb}{0.000000, 0.000000, 0.000000}
\pgfsetstrokecolor{dialinecolor}
\node[anchor=east] at (17.250000\du,0.521559\du){$x$};
\definecolor{dialinecolor}{rgb}{0.000000, 0.000000, 0.000000}
\pgfsetstrokecolor{dialinecolor}
\node at (28.250000\du,2.971589\du){Latent space};
\definecolor{dialinecolor}{rgb}{0.000000, 0.000000, 0.000000}
\pgfsetstrokecolor{dialinecolor}
\node at (16.800000\du,2.971589\du){Data space};
\pgfsetlinewidth{0.050000\du}
\pgfsetdash{}{0pt}
\pgfsetdash{}{0pt}
\pgfsetroundjoin
{\pgfsetcornersarced{\pgfpoint{0.500000\du}{0.500000\du}}\definecolor{dialinecolor}{rgb}{1.000000, 1.000000, 1.000000}
\pgfsetfillcolor{dialinecolor}
\fill (27.000000\du,0.000000\du)--(27.000000\du,1.000000\du)--(29.500000\du,1.000000\du)--(29.500000\du,0.000000\du)--cycle;
}{\pgfsetcornersarced{\pgfpoint{0.500000\du}{0.500000\du}}\definecolor{dialinecolor}{rgb}{0.000000, 0.000000, 0.000000}
\pgfsetstrokecolor{dialinecolor}
\draw (27.000000\du,0.000000\du)--(27.000000\du,1.000000\du)--(29.500000\du,1.000000\du)--(29.500000\du,0.000000\du)--cycle;
}
\definecolor{dialinecolor}{rgb}{0.000000, 0.000000, 0.000000}
\pgfsetstrokecolor{dialinecolor}
\node at (28.250000\du,0.521559\du){$q(z|x)$};
\pgfsetlinewidth{0.050000\du}
\pgfsetdash{}{0pt}
\pgfsetdash{}{0pt}
\pgfsetroundjoin
{\pgfsetcornersarced{\pgfpoint{0.500000\du}{0.500000\du}}\definecolor{dialinecolor}{rgb}{1.000000, 1.000000, 1.000000}
\pgfsetfillcolor{dialinecolor}
\fill (15.500000\du,3.500000\du)--(15.500000\du,4.500000\du)--(18.000000\du,4.500000\du)--(18.000000\du,3.500000\du)--cycle;
}{\pgfsetcornersarced{\pgfpoint{0.500000\du}{0.500000\du}}\definecolor{dialinecolor}{rgb}{0.000000, 0.000000, 0.000000}
\pgfsetstrokecolor{dialinecolor}
\draw (15.500000\du,3.500000\du)--(15.500000\du,4.500000\du)--(18.000000\du,4.500000\du)--(18.000000\du,3.500000\du)--cycle;
}
\definecolor{dialinecolor}{rgb}{0.000000, 0.000000, 0.000000}
\pgfsetstrokecolor{dialinecolor}
\node at (16.750000\du,4.021559\du){$p(x|z)$};
\pgfsetlinewidth{0.050000\du}
\pgfsetdash{}{0pt}
\pgfsetdash{}{0pt}
\pgfsetroundjoin
{\pgfsetcornersarced{\pgfpoint{0.500000\du}{0.500000\du}}\definecolor{dialinecolor}{rgb}{1.000000, 1.000000, 1.000000}
\pgfsetfillcolor{dialinecolor}
\fill (27.000000\du,5.500000\du)--(27.000000\du,6.500000\du)--(29.500000\du,6.500000\du)--(29.500000\du,5.500000\du)--cycle;
}{\pgfsetcornersarced{\pgfpoint{0.500000\du}{0.500000\du}}\definecolor{dialinecolor}{rgb}{0.000000, 0.000000, 0.000000}
\pgfsetstrokecolor{dialinecolor}
\draw (27.000000\du,5.500000\du)--(27.000000\du,6.500000\du)--(29.500000\du,6.500000\du)--(29.500000\du,5.500000\du)--cycle;
}
\definecolor{dialinecolor}{rgb}{0.000000, 0.000000, 0.000000}
\pgfsetstrokecolor{dialinecolor}
\node at (28.250000\du,6.021559\du){$p_\theta(z)$};
\pgfsetlinewidth{0.100000\du}
\pgfsetdash{}{0pt}
\pgfsetdash{}{0pt}
\pgfsetbuttcap
{
\definecolor{dialinecolor}{rgb}{0.000000, 0.000000, 0.000000}
\pgfsetfillcolor{dialinecolor}
\pgfsetarrowsend{stealth}
\definecolor{dialinecolor}{rgb}{0.000000, 0.000000, 0.000000}
\pgfsetstrokecolor{dialinecolor}
\draw (28.250000\du,5.500000\du)--(28.250000\du,4.400000\du);
}
\pgfsetlinewidth{0.100000\du}
\pgfsetdash{}{0pt}
\pgfsetdash{}{0pt}
\pgfsetmiterjoin
\pgfsetbuttcap
{
\definecolor{dialinecolor}{rgb}{0.000000, 0.000000, 0.000000}
\pgfsetfillcolor{dialinecolor}
\pgfsetarrowsend{stealth}
{\pgfsetcornersarced{\pgfpoint{0.000000\du}{0.000000\du}}\definecolor{dialinecolor}{rgb}{0.000000, 0.000000, 0.000000}
\pgfsetstrokecolor{dialinecolor}
\draw (29.500000\du,0.500000\du)--(31.000000\du,0.500000\du)--(31.000000\du,4.000000\du)--(29.500000\du,4.000000\du);
}}
\end{tikzpicture}}
    \caption{VAE architecture: the encoder (resp.  decoder) module is a neural network taking $x$ (resp $z$) as input and computing the parameters of distribution $q(z|x)$ (resp., $p(x|z)$). }
    \label{fig:vae}
\end{figure}
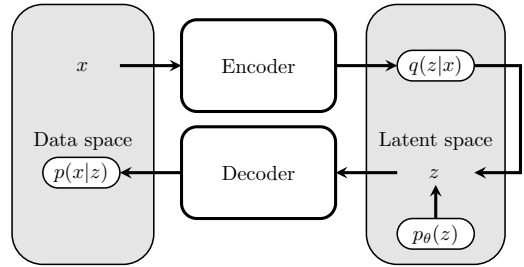

The maximization of the data likelihood (Eq. (\ref{eq:generative-model})) does not prevent the observation model from generating out-of-distribution samples \cite{theis_note_2016}. In practice, VAEs are observed to generate unrealistic samples more often than Generative Adversarial Networks \cite{radford_unsupervised_2015}, prompting the introduction of adversarial losses in VAEs (see for instance  \cite{larsen_autoencoding_2016,dosovitskiy_generating_2016}).

\paragraph{Inference model.}
The role of the inference model is to enforce the quality of the lower bound (Eq. (\ref{eq:ELBO})). In the (ideal) equality case, $q(z|x) = p_\theta(z|x)$, that is, the inference model is a probabilistic inverse of the observation model relative to the prior $p(z)$. In such a case, the dataset would be perfectly represented (compressed) in the latent space, with $q(z|x)$ and  $p_\theta(x|z)$ potentially being very complex maps. When the inference model is poor, the ELBO is quite loose, the input space is mapped onto a poorly compressed latent space, increasing the risk of generating out-of-distribution samples. 

As said, the inference model is classically built upon a Normal distribution through the reparametrization trick, while prior model $p(z)$ is set to $\mathcal{N}(0;\mathbf{I})$. The limitations of this setting are related to its smoothness, that might not reflect the data, opening mostly two directions for VAE extensions. 

In the case where the dataset is structured, e.g. involving disjoint clusters, the problem is that Gaussian encoders cannot build sharp boundaries in the latent space. A natural option thus is to also create clusters in the latent space to ensure a good reconstruction, via a more complex prior. More complex distribution classes, e.g. involving categorical variables, are considered to handle structures in the dataset \cite{kingma_semi-supervised_2014,kingma_improved_2016}, as well as auto-regressive priors \cite{nalisnick_stick-breaking_2017,van_den_oord_neural_2017,razavi_generating_2019}.

Another direction for VAE extensions stems from the following remark. As the decoder defines a continuous map from the latent $z$ to the data $x$, it is necessary that the latent space topology jointly defined by $q(z|x)$ and $p(z)$ matches the dataset topology. This remark calls for using hierarchical latent structures \cite{so_nderby_ladder_2016} or exploring non-Euclidean latent spaces \cite{nagano_wrapped_2019,mathieu_continuous_2019}.

\paragraph{Observation model.}
Besides a good inference model $q(x|z)$, the VAE success also requires a good observation model $p_\theta(x|z)$, efficiently untangling a compressed $q(x|z)$ and mapping it back onto the data space.

The choice of the distribution class for the observation model governs the reconstruction loss ($-\log p_\theta(x|z)$), with a significant impact on the quality of the generated samples. Typically, the use of a Normal distribution class with fixed variance boils down to using a mean square error in the input space. Such a model is clearly inappropriate in the domain of computer vision, as setting an independent noise on the image pixels hardly yields a realistic image.

Two main research directions have been investigated to improve the observation model while keeping computational complexity under control. The former one leverages domain knowledge to augment the VAE criterion, e.g. using perceptual metrics \cite{larsen_autoencoding_2016,dosovitskiy_generating_2016}.
The latter one considers significantly more complex distribution classes
\cite{gulrajani_pixelvae_2016,sadeghi_pixelvae_2019},  e.g. involving autoregressive models over the pixels. A difficulty with such models is that their representation power is sufficient to learn the data distribution while ignoring the latent space. This phenomenon, referred to as posterior collapse \cite{he_lagging_2019,alemi_fixing_2017}, is blamed (in the linear VAE case) on the presence of spurious optima in the log marginal likelihood
\cite{lucas_dontextquotesingle_2019}.

\section{Impact of the observation model}
\label{sec:th}
The claim that the most severe cause of VAE failures is a bad approximation of the data manifold (Pb (1)) is argued as follows. On the one hand, the samples generated by the VAE can hardly be of better quality than the samples reconstructed from the data; and the reconstructed samples can only be as good as the manifold approximation.
On the other hand, the approximation of the data distribution on the manifold (Pb (2)) is directly addressed by the ELBO criterion (Eq (\ref{eq:ELBO})), accounting for the diverse modes of the data; actually, VAEs do not suffer from the notorious mode-dropping phenomenon observed in GANs  \cite{arjovsky_towards_2017}.

As said, this paper focuses on one among the possible causes for the bad approximation of the data manifold, namely the choice of the distribution family $p_\theta$. 
This section first presents a case of non-identifiability of the data manifold, only due to the choice of the distribution family $p_\theta$. Some approaches used in the literature to avoid this failure case are discussed and an alternative is proposed.

\subsection{A non-identifiability case}\label{sec:non-id}
Let us consider the usual choice of a Normal distribution, where the  observation model $p(x | z)$
is defined as ${\cal N}(\mu(z), \sigma^2 \mathbf{I})$ with $\mu(z)$ the output of the decoder network and $\sigma$ a fixed scalar variance. While in this setting the reconstruction term conveniently boils down to the squared error loss, the choice of $\sigma$ has a dramatic impact on the identification of the data manifold. 
Let us consider a dataset located on the hyper-sphere of radius $R$ in dimension $D$. The following result establishes that a VAE cannot capture the data manifold for large values of hyper-parameter $\sigma$:

\begin{theorem}
Let us assume arbitrarily powerful neural architectures for the encoder and decoder networks, as well as an arbitrarily powerful latent space.
Let us further assume that the data is an infinite sample size of the $D$-dimensional hyper-sphere of radius $R$ noted  $\mathbf{S}_{D-1}(R)$. 

Then, a VAE with observation model ${\cal N}(\cdot,\sigma^2 \mathbf{I})$ fails to characterize the data manifold $\mathbf{S}_{D-1}(R)$ if 
$$\sigma \geq \frac{R}{\sqrt{D-1}}$$
\end{theorem}

\begin{proof}

Let $\mu(z)$ denote the output of the decoder network of the VAE. By definition, $p_\theta(x|z) = \mathcal{N}(\mu(z), \sigma^2 \mathbf{I})$.
Under the large sample and arbitrarily powerful neural architecture assumptions, the VAE can exactly characterize $q(z|x)$ and $p(z)$. Distribution $p(z)$ can be replaced by the distribution over the (deterministic) decoder output $p(\mu)$. The ELBO equation therefore becomes an equality and $p_\theta(x)$ can be computed exactly:

\begin{equation}
    p_\theta(x) = \int p(x | \mu) p_\theta(\mu) d\mu
    \label{eq:ptheta}
\end{equation}
with
\begin{equation}
p(x|\mu) = \frac{1}{(\sqrt{2\pi} \sigma)^D} \exp\left(-\frac{1}{2 \sigma^2} \|x - \mu\|^2\right)
\end{equation}
Under the assumptions made, the optimal distribution $p_\theta(\mu)$ is derived as follows.
The data being radially symmetrical, the resulting distribution will necessarily be so as well. Distribution $p_\theta(\mu)$ thus only depends on the distance $r$ from $\mu$ to the center of the sphere. Furthermore, $p_\theta(x)$ being the expectation of function $p(x|\mu)$ under $p_\theta(\mu)$, it is maximized when $p_\theta(\mu)$ has all its mass where $p(x|\mu)$ reaches its maximum. Accordingly,  $p_\theta(\mu)$ ranges among the uniform distributions over some $S_{D-1}(r)$,
hyper-sphere of some radius $r$ in dimension $D$  
with same center as the dataset. Taking normalization into account, it comes:
\begin{equation}
    p_\theta(\mu) = \frac{1}{r^{D-1} A_{D-1}} \mathbf{1}_{(\|\mu\| = r)}
\end{equation}
with $A_{D-1}$ the area of a unit $D$-dimensional sphere. Using spherical coordinates, 
\begin{equation}
\begin{array}{c}
    \|x - \mu\|^2  = R^2 + r^2 - 2 r R \cos(\phi)\\
    d\mu = A_{D-2} r^{D-1} \sin^{D-2}(\phi) dr d\phi
\end{array}
\end{equation}
with $\phi$ the angle between $x$ and $\mu$ wrt the center of the sphere and $A_{D-2}$ the area of the unit hyper-sphere in dimension $D-2$ (over which the integration is trivial as it has no impact on other terms).

Removing constant terms wrt to $r$ and $\phi$ it comes:
\begin{equation}
    p_\theta(x) \propto e^{-\frac{r^2}{2 \sigma^2}} \int_{\phi=0}^{\pi} sin^{D-2}(\phi) \exp\left(\frac{r R}{\sigma^2} \cos(\phi)\right) d\phi
    \label{eq:pthetaprop}
\end{equation}

It is easily shown that the derivative of $p_\theta(x)$ has same sign as:
\begin{equation}
    \int_{0}^{\pi} \frac{r}{\sigma^2} \left(\frac{R^2 \sin^2(\phi)}{(D-1)\sigma^2} - 1\right)\sin^{D-2}(\phi) e^{rR\cos(\phi)} d\phi
    \label{eq:pthetaderiv}
\end{equation}
Accordingly, an extremum is reached for $r=0$, and $r=0$ is the global maximum for $\frac{R^2}{(D-1)\sigma^2} \leq 1$. 
Therefore, if the radius $R$ of the data hyper-sphere is small compared to standard deviation $\sigma$ ($R \leq \sigma \sqrt{D-1}$), in the most favorable case of large sample limit and arbitrary power of the neural architectures, the VAE can but approximate the data manifold by a normal distribution with same center as the dataset hypersphere. In other words, it grossly characterizes the volume of the hyper-sphere instead of its area. 
\end{proof}

Inspecting the second order derivative of Eq (\ref{eq:pthetaprop}) shows that while $r=0$ is a local optimum, it is a local maximum for $R < \sigma \sqrt{N}$, and a local minimum for larger values of $R$. In the latter case, the optimal value of $r$ asymptotically converges toward $R=r$. In other words, the generative model $p_\theta(x)$ is expressed as a mixture of Normal distributions, the mean of which is located on the data manifold: the generative model is accurately paving the data manifold. 

For $D=2$, Eq. (\ref{eq:pthetaprop}) can be computed analytically; its heat map depending on  $\frac{R}{\sigma}$ and $\frac{r}{\sigma}$ is depicted in Fig. \ref{fig:contourplot}. For small values of $\frac{R}{\sigma}$, $r=0$ is the only global maximum. As $\frac{R}{\sigma}$ increases, the maximum moves toward the $r=R$ diagonal; graphically, this change occurs for $\frac{R}{\sigma} \approx \sqrt{2}$.

\begin{figure}[htb]
    \centering
    \includegraphics[width=.9\columnwidth]{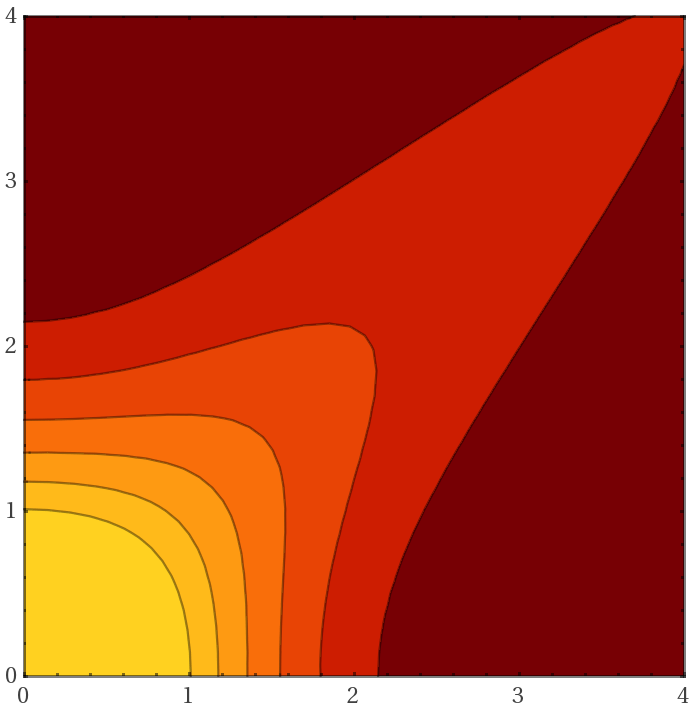}
    \caption{Heat map of $p_\theta(x)$ vs $\frac{R}{\sigma}$ (horizontal axis) and $\frac{r}{\sigma}$ (vertical axis) for D=2 (the lighter the higher; better seen in color). See comments in text.}
    \label{fig:contourplot}
\end{figure}

\subsection{External control of the model sharpness}\label{sec:ext}
Most generally, the sharpness of the distribution class governs the information extracted from the data. For a given fixed sharpness, the observation model is bound to ignore all details that are small relatively to this sharpness: the gain in terms of reconstruction loss is smaller than the cost incurred by encoding these details in the latent variables, regardless of the data amount and of the representation power of the neural encoder and decoder. 

The sharpness impact is often hidden by the fact that generated samples are derived as the expectation of distribution $p_\theta(x|z)$, particularly so in the case of computer vision \cite{larsen_autoencoding_2016,dosovitskiy_generating_2016,higgins_beta-vae:_2017}: the generated sample simply is the raw output of the decoder, that is the mean\footnote{The interpretation is more complex for other losses such as the Bernoulli loss \cite{loaiza-ganem_continuous_2019}. } of $p_\theta(x|z)$. This generation process indeed mitigates the impact of large $\sigma$ values; still, it does not address the potential mis-specifications of the observation model. In the particular case of images, the squared error loss corresponds to a Normal distribution with variance $\sigma^2 = \frac{1}{2}$. This very large variance (considering that pixels are usually normalized in the $[0,1]$ interval) causes most data details to be discarded, explaining the blurriness of the generated images.\\
Most generally, the VAE learning criterion aims at a generative process sampling (as opposed to, averaging) $p_\theta(x|z)$. The averaging strategy thus creates some discrepancy between the well-founded VAE design and its actual usage. 

Another strategy is to reconsider the VAE training loss itself and design a more effective one, be it derived from a principled probabilistic setting or not. The simplest option is to control the trade-off between the reconstruction loss 
and the latent compression by means of some penalization weight $\lambda$ on the latent term.
In the general case, increasing the weight of the latent term cannot be rigorously interpreted in terms of the sharpness of the observation model due to normalization issues,\footnote{The importance of taking into account such normalization issues, and the impact of neglecting those are discussed in  \cite{loaiza-ganem_continuous_2019}.} but informally speaking, introducing a $\lambda$ term in front of the latent term of the loss is similar to raising the observation model distribution at the power $\frac{1}{\lambda}$: depending on the value of $\lambda$ the model might capture finer-grained details of the data and increase the latent loss ($\lambda < 1$); or discard the details and achieve a higher compression of the latent variables ($\lambda > 1$). \\
In the case of a Normal distribution with fixed variance $\sigma^2$, using a penalization weight $\lambda$ on the latent loss exactly corresponds to multiplying $\sigma$ by $\sqrt{\lambda}$.

\subsection{Learning the  model sharpness}
As the VAE success depends on a good approximation of the data manifold,
and the quality of the approximation depends in part on the observation model, one possibility is to allow this model to be arbitrarily sharp, e.g. by making the variance of the Normal distribution a learned parameter, thus defining the L-VAE setting.\footnote{This setting must be distinguished from the well-known re-parameterization trick \cite{goodfellow_generative_2014}, operating on the latent distribution $q(z|x)$.} The reconstruction term of the L-VAE loss then reads (up to an additive constant), with $D$ the dimension of the input space:
\begin{equation}
    - \log p_\theta(x | z) = \frac{1}{2 \sigma^2}\|x-\mu(z)\|^2 + D \log \sigma
    \label{eq:lsigma}
\end{equation}
Maximizing the likelihood of the generated samples based on Eq. (\ref{eq:lsigma}) allows the L-VAE to autonomously adjust the trade-off between the reconstruction and the latent losses (as opposed to, manually setting a fixed  penalization weight on the latent term). As will be seen (section \ref{sec:dyn}), the desirable balance between both terms evolves along the learning trajectory, making all the more important to let the VAE control the trade-off.

Normal distributions with learned $\sigma$ permit in principle an arbitrarily accurate approximation of the manifold, by selecting arbitrarily small values of $\sigma$, while the $D \log \sigma$ term (Eq (\ref{eq:lsigma})) is an incentive to choose $\sigma$ as small as possible, to the extent permitted by the 
rest of the model.

A straightforward refinement of the above isometric Normal distribution is to consider a diagonal covariance matrix ${\cal N}(\mathbf{\mu}(z), \mathbf{\sigma}(z))$, with $\mathbf{\sigma}$ a $D$-dimensional vector computed by the decoder akin $\mathbf{\mu}$. The associated training loss reads:
\begin{equation}
\label{eq:learned-multi-sigma}
    - \log p_\theta(x | z) = \sum_{i=1}^D \frac{1}{2 \sigma_i(z)^2} (x_i - \mu_i(z))^2 + \log \sigma_i(z)
\end{equation}

This refinement might still be insufficient when the data manifold is not aligned with the axes of the input space data, requiring a Normal distribution with full covariance matrix to accurately describe the manifold. Learning the full covariance matrix however raises strong scalability issues, 
and only the scalar and vectorial L-VAE settings (Eqs. (\ref{eq:lsigma}) and (\ref{eq:learned-multi-sigma})) are considered in the following.

\section{Inspecting the VAE dynamics}\label{sec:dyn}

This section investigates in more depth the impact of the observation model on the generative process. An artificial dataset, used to provide empirical evidence and backup our first claim, enables to inspect the VAE learning  dynamics. The main lesson learned, and the second claim of the paper, is that learning the sharpness of the observation model allows the VAE to escape bad local optima of the optimization landscape.
The claims and the proposed interpretation of the results on the artificial dataset are confirmed by experiments on MNIST and CelebA \cite{liu_deep_2015}.

\subsection{Experimental setting}
The main goal of experiments is to assess the impact of the observation model and its sharpness, be it fixed or learned, on the reconstruction and generation abilities of a VAE. 

An artificial problem is designed in $\mathbf{R}^2$  (Fig. \ref{fig:syn-dataset}): data are located on a 1D manifold,  structured in three regions with moderate, small and large noise. This design aims to independently assess whether the considered VAE settings enable to identify the manifold and its fine-grained details.

\begin{figure}
    \centering
    \includegraphics[width=0.9\columnwidth]{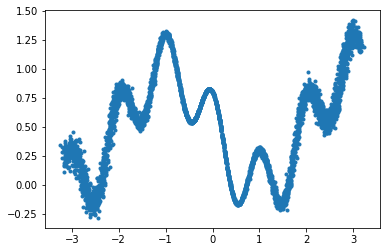}
    \caption{Artificial dataset in $\mathbf{R}^2$, located on a 1D manifold with heteroscedastic noise.}
    \label{fig:syn-dataset}
\end{figure}

The neural architecture used for both the encoder and the decoder is a ResNet \cite{he_deep_2016} with 4 residual blocks of 200 neurons. Considering the data complexity, this architecture is expected to offer a "sufficient" representation capacity (satisfying the assumptions of Thm 1), and make the identification of the data manifold the only bottleneck for VAE. Likewise, the considered dataset satisfies the large sample limit assumption, with 100 points sampled anew from the data distribution in each batch.

Several VAE settings are experimentally compared, only differing in the observation model ${\cal N}(\cdot,\sigma)$, with no penalization weight on the latent term:\\
i) F-VAE considers a fixed scalar  
$\sigma$, ranging in  $\{1.0, 0.3, 0.1, 0.03, 0.01 \}$; ii) LS-VAE considers a learned scalar $\sigma$ (initialized as $\sigma = 1$); iii) LV-VAE considers a learned vector $\sigma(z)$, as an output of the decoder network; iv) finally, F$^*$-VAE considers a fixed scalar $\sigma$ set to the final value learned by LS-VAE.

The first performance indicator is
the ELBO value\footnote{
The ELBO values reported for F-VAE are augmented with a $2\log \sigma$ term, for a fair comparison with LS-VAE and LV-VAE, after Eqs (\ref{eq:lsigma}) and (\ref{eq:learned-multi-sigma}).
}, quantitatively assessing the generative model. Two other qualitative performance indicators are proposed: The reconstruction performance is qualitatively assessed by displaying the reconstructed samples (the mean of $p_\theta(x|z)$) on the top of the initial samples; a perfect reconstruction is characterized by the fact that the reconstructed samples exactly hide the initial ones. Likewise, the generative performance is qualitatively assessed by comparing the generated samples
with the initial dataset.

\subsection{Results and analysis}
\begin{table}
    \centering
    \begin{tabular}{l|c|c}
      Setting &  $\sigma$ & ELBO \\
        \hline
     F-VAE, Fixed $\sigma$ &   $\sigma = 1.0$ & $-1.14 \pm 0.02$ \\
       &  $\sigma = 0.3$ & $-0.063 \pm 0.03$ \\
       &  $\sigma = 0.1$ & $0.773 \pm 0.09$ \\
       &  $\sigma = 0.043$ & $0.335 \pm 0.08$ \\
       &  $\sigma = 0.03$ & $0.141 \pm 0.10$ \\
       &  $\sigma = 0.01$ & $-0.961 \pm 0.12$ \\ \hline
       LS-VAE, scalar $\sigma$ & $\sigma^* = 0.043$ &  $\mathbf{1.26 \pm 0.09}$ \\ \hline
       LV-VAE, vector $\sigma$ & & 
$\mathbf{1.62 \pm 0.08}$
    \end{tabular}
    \caption{ELBO values (averaged over 10 runs) of F-VAE, LS-VAE and LV-VAE; the higher the better.
    }
    \label{tab:syn-elbos}
\end{table}

Table \ref{tab:syn-elbos} displays the ELBO indicator for all considered VAE settings (averaged on 10 runs). These results show the decisive impact of learning the variance of the observation model: the option significantly dominating all others is LV-VAE, learning a vector $\sigma$; the second best is LS-VAE, learning a scalar $\sigma$.

Most interestingly, when fixing $\sigma$ to the optimal scalar value found by LS-VAE ($\sigma = .043$) the ELBO obtained by F$^*$-VAE is significantly  degraded compared to LS-VAE; the ELBO of F$^*$-VAE also is significantly lower than for F-VAE with $\sigma = 0.1$.
These results suggest that the dynamics of the learning process matters as much as the actual sharpness $\sigma$ of the observation model. 

A more detailed analysis is permitted by the qualitative performance indicators, visually displayed in Fig. \ref{fig:syn-sigmas}. 
Regarding F-VAE, the reconstruction  is poor for high values of $\sigma$ (Fig. \ref{fig:syn-sigmas}.(a), Left; all details are lost, following the analysis in section \ref{sec:th}); as $\sigma$ decreases, the reconstruction gradually improves and the reconstructed samples exactly hide the initial ones for $\sigma \le .03$ (at the expense of the compression and generation quality). Note that for medium values of $\sigma$, the reconstructed samples are closer to the manifold than the initial ones, particularly so in the regions with high or moderate noise (Fig. \ref{fig:syn-sigmas}.(b), Left). This phenomenon is likewise explained by the fact that the VAE discards details that are small comparatively to the model sharpness (section \ref{sec:th}): when the data noise is smaller than $\sigma$ (e.g. for $\sigma=.1$), the noise is discarded, i.e. the reconstruction operates like a denoiser.\\
Regarding F-VAE, the generation  is poor in all cases but for $\sigma=.1$ (Fig. \ref{fig:syn-sigmas},(a-c), Right). For high values of $\sigma$, the details are missed, with a roughly 
cosine-like shape generated for 
$\sigma = .3$. For small values of $\sigma$, the generated samples also form a rough shape with many samples far from the manifold, all the more so as $\sigma$ decreases. Despite the blurriness of the generated samples, their average does accurately capture the data manifold. This situation is similar to that observed on, e.g., the CelebA dataset, where the generated samples are seen as random noise, while their average rightly resembles a face. 

Regarding L-VAE, in both scalar and vector cases the reconstruction is good although the denoising effect is observed (reconstructed samples are closer to the manifold), which is attributed to the comparatively large value of the learned $\sigma$ (Fig. \ref{fig:syn-sigmas},(d),(f), Left). In both scalar and vector cases, the generation is quite good compared to that of F-VAE at its best. In the scalar case, the generated samples are located on a single-width band centered on the manifold: the observation model does not have the representation power needed to fit the varying data noise. In the vector case, the observation model allows to encode both the manifold and its thickness, and the generated samples exactly match the initial samples (perfect generation). 

In order to disentangle the impacts of the model sharpness and of the optimization process, additional experiments are conducted with a fixed-$\sigma$ VAE, where $\sigma$ is set to the value learned by FS-VAE ($\sigma = .043$). The reconstruction and generation results are displayed on Fig. \ref{fig:syn-learned-sigmas}.(e). The reconstruction is decent, with no denoising effect as $\sigma$ is sufficiently small to capture the details. The generation is much poorer than for L-VAE, and even poorer than for F-VAE with  $\sigma = .1$, with the generated samples located on a thick band centered on the manifold. 

This comparison suggests that besides the sharpness of the observation model, the second key issue for the VAE performance is the trade-off between the reconstruction and the latent losses, and how this trade-off is dynamically controlled along the learning trajectory. It is observed that FS-VAE decreases the learned $\sigma$ value by jumping from time to time to lower $\sigma$ plateaus, and these jumps coincide with leaps of the ELBO (more in appendix).

The dynamics of the learned $\sigma$ along the VAE training can be interpreted in terms of an annealing process, taking inspiration from the analysis of \cite{huang_improving_2018}. In the early learning stage, the large $\sigma$ value compensates for the large reconstruction loss and lowers the energy landscape. At this stage, the latent compression term of the loss dominates, and the model only aims to match the global shape of the data manifold. As the reconstruction at the large scale improves, the model can then afford to lower $\sigma$ and improve the overall loss. This has the effect of gradually making the energy landscape steeper, raising more details of the data to the attention of the model, and allowing $\sigma$ to be again lowered as the reconstruction loss decreases. The process continues until any improvement of the reconstruction term would be
cancelled out by a degradation of the latent term. At this point $\sigma$ stabilizes and the self-controlled annealing ends. This gradual evolution of the energy landscape allows the VAE to iteratively build and refine its latent representation, eventually yielding a better compression than the one that would be reached by running the optimization in a fixed energy landscape.

Complementary experiments (in appendix) show that an externally set $\sigma$ schedule yields the same results (quantitative ELBO, visual reconstruction and generation) as the LS-VAE, confirming the importance of a dynamic schedule to adjust the reconstruction {\em vs} latent optimization trade-off.

\begin{figure}
    \centering
    \subfloat[$\sigma = 0.3$]{
        \includegraphics[width=0.45\columnwidth]{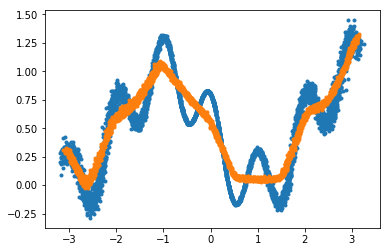}
        \includegraphics[width=0.45\columnwidth]{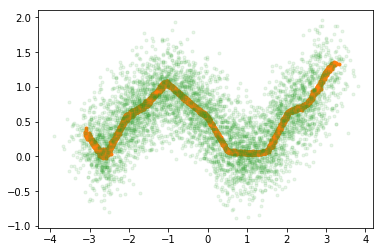}
    }
    
    \subfloat[$\sigma = 0.1$]{
        \includegraphics[width=0.45\columnwidth]{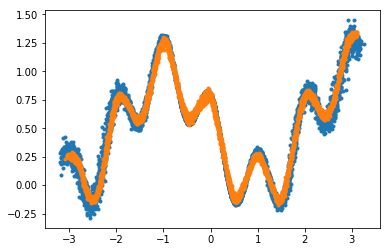}
        \includegraphics[width=0.45\columnwidth]{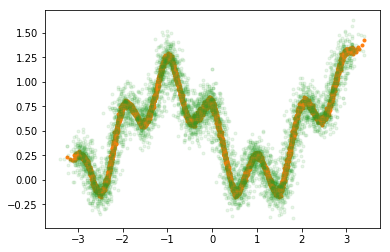}
    }
    
    \subfloat[$\sigma = 0.01$]{
        \includegraphics[width=0.45\columnwidth]{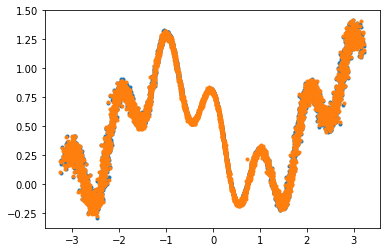}
        \includegraphics[width=0.45\columnwidth]{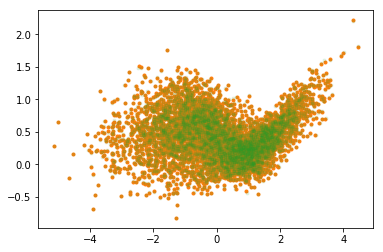}
    }

   \subfloat[Learned scalar $\sigma$]{
        \includegraphics[width=0.45\columnwidth]{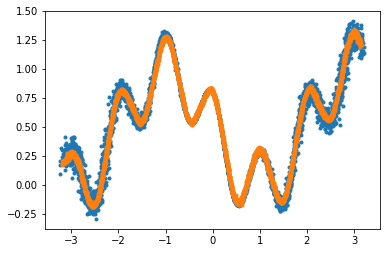}
        \includegraphics[width=0.45\columnwidth]{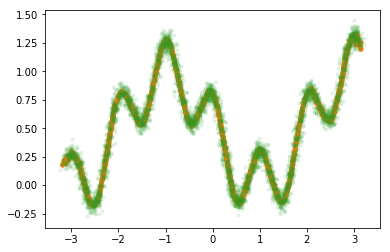}
    }

    \subfloat[$\sigma = 0.043$]{
        \includegraphics[width=0.45\columnwidth]{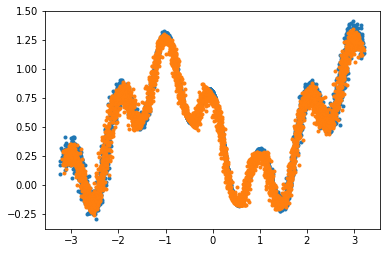}
        \includegraphics[width=0.45\columnwidth]{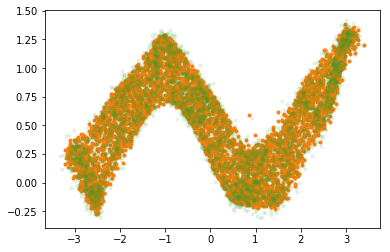}
    }

    \subfloat[Learned vectorial $\sigma$]{
        \includegraphics[width=0.45\columnwidth]{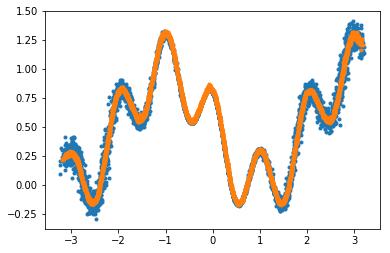}
        \includegraphics[width=0.45\columnwidth]{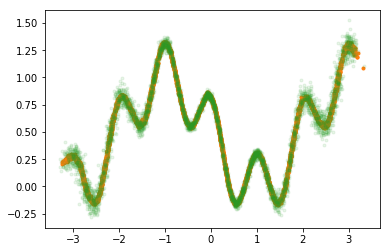}
    }

    \caption{(a)-(c): F-VAEs performance with fixed $\sigma$ ranging from 1 to .01 from top to bottom ; (d): LS-VAE; (e): F-VAE with same $\sigma$ value as LS-VAE; (f): LV-VAE (better seen in color). 
    \newline
    Left: Reconstructed samples (in orange) and initial samples (in blue). Right: Generated samples (in green) and means of $p_\theta(x|z)$ (in orange).}
    \label{fig:syn-sigmas}
    \label{fig:syn-learned-sigmas}
\end{figure}

\subsection{Image datasets}

The significance of the above discussion is examined by comparing the F-VAE and LV-VAE settings on the MNIST and CelebA datasets. 

Note that a usual practice in computer vision is to set the reconstruction loss to the mean square error, which is equivalent to considering a fixed Normal distribution with variance $\sigma^2 = 1/2$. 
As said, this large variance is expected to prevent the reconstruction from capturing the fine-grained details and to cause generated samples to look like random ones (all things considered, the situation is similar to that of $\sigma = .3$ in the 1D artificial problem). The generative process thus can only fall back on the average of the generated samples, that is realistic although blurry (Fig. \ref{fig:vae-rec-gen}.(a-b), Left). 

\begin{table}
    \centering
    \begin{tabular}{c|c|c}
        Dataset & F-VAE & LV-VAE \\
        \hline
        MNIST & $0.06$ & $\mathbf{4.89}$ \\
        CelebA & $5.52$ & $\mathbf{57.51}$ \\
    \end{tabular}
    \caption{Comparison of F-VAE (mean square error) and LV-VAE on MNIST and CelebA: ELBO (in bits per pixel, accounting for the $\log\sigma$ term for a fair comparison). }
    \label{tab:img-elbo}
\end{table}

\begin{figure}[!tbh]
    \centering
    \subfloat[MNIST reconstruction]{
        \adjincludegraphics[width=0.45\columnwidth, trim={0 0 {.5\width} 0},clip]{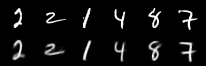}
        \adjincludegraphics[width=0.45\columnwidth, trim={0 0 {.5\width} 0},clip]{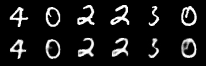}
    }
    
    \subfloat[CelebA reconstruction]{
        \adjincludegraphics[width=0.45\columnwidth, trim={0 0 {.5\width} 0},clip]{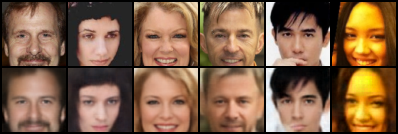}
        \adjincludegraphics[width=0.45\columnwidth, trim={0 0 {.5\width} 0},clip]{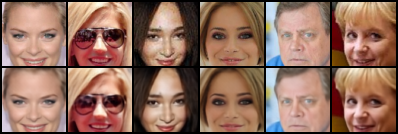}
    }
    
    \subfloat[MNIST generation]{
        \adjincludegraphics[width=0.45\columnwidth, trim={0 0 {.5\width} 0},clip]{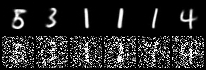}
        \adjincludegraphics[width=0.45\columnwidth, trim={0 0 {.5\width} 0},clip]{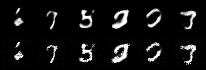}
    }
    
    \subfloat[CelebA generation]{
        \adjincludegraphics[width=0.45\columnwidth, trim={0 0 {.5\width} 0},clip]{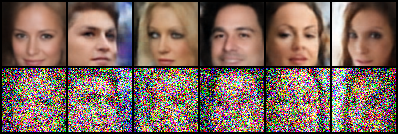}
        \adjincludegraphics[width=0.45\columnwidth, trim={0 0 {.5\width} 0},clip]{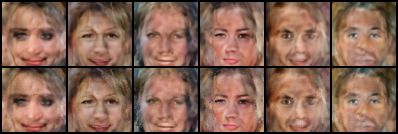}
    }
    
    \caption{Reconstruction and generation on MNIST and CelebA.
    \newline
    (a)-(b): Reconstruction; Top: initial images. Bottom Left: F-VAE (MSE loss); Bottom Right: LV-VAE (Eq. 11).
    \newline
    (c)-(d): Generation; Left: F-VAE (MSE loss), Right: LV-VAE (learned vectorial $\sigma$). Means of $p_\theta(x|z)$ are shown in the top rows; samples are shown in the bottom rows.}
    \label{fig:vae-rec-gen}
\end{figure}

The F-VAE and LV-VAE approaches are compared in terms of reconstruction on Fig. \ref{fig:vae-rec-gen}.(a-b). As LV-VAE controls the observation sharpness, it achieves a quasi pixel-perfect  reconstruction (Fig. \ref{fig:vae-rec-gen}.(a-b), Right), much better than F-VAE. The quantitative ELBO indicator (Table \ref{tab:img-elbo}) confirms that F-VAE is dominated by LV-VAE by about one order of magnitude.\footnote{As in Table \ref{tab:syn-elbos}, the ELBO values reported for F-VAE are augmented with a $\log \sigma$ term for the sake of a fair comparison.}

Regarding the generation performance, the samples $p_\theta(x | z)$ generated by F-VAE are viewed as random noise (Fig. \ref{fig:vae-rec-gen}.(c-d), Left, bottom), but their averages (the mean of $p_\theta(x|z)$) are quite similar to the inverse image of $z$, suggesting that while the VAE identifies a much simplified manifold, its latent space manages to efficiently compress this manifold (Fig. \ref{fig:vae-rec-gen}.(c-d), Left, top). Quite the contrary, the samples generated by LV-VAE (Fig. \ref{fig:vae-rec-gen}.(c-d), Right, bottom) are much better than
random images; the precision is similar to that of reconstructed images. Henceforth, the samples and their averages are very similar; unfortunately they are less convincing than the averaged samples generated by F-VAE (some of the MNIST digits are distorted, and the CelebA faces lack structure). 
A tentative interpretation for this weakness is as follows. As  LV-VAE approximates the data manifold with a better accuracy, more details need be represented, making the latent space more difficult to compress. Eventually the LV-VAE faces another bottleneck: the insufficient power of representation of the latent space.\footnote{This representation power can be limited by both the capacity of the encoder and decoder neural networks, as well as the probabilistic structure chosen for the latent variable $z$, here set to the traditional factorized Gaussian.}.

\section{Conclusion and Perspectives}
The main contributions of the paper is to shed more light on the VAE bottlenecks, to propose some principles in order to address these bottlenecks, and to illustrate the practical efficiency thereof.

Our first claim is that the primary VAE bottleneck lies in the identification of the data manifold. In particular, an observation model unable to pave the data manifold is bound to grossly approximate the data and discard important information, irrespective of the expressive power of the other VAE elements. 

Our second claim is that a good identification of the manifold can be achieved in a theoretically sound and algorithmically efficient way: the proposed L-VAE learns the sharpness of the observation model, thereby dynamically adjusting the optimization trade-off between the reconstruction and latent terms akin a simulated annealing schedule, ending up with a richer and better compressed latent space than allowed with a fixed sharpness. The L-VAE achieves a quasi pixel-perfect reconstruction, and removes the generation noise associated with a too large fixed variance. In counterpart, as the L-VAE considers more information, it exerts more pressure on the latent space and its compression capabilities, decreasing the visual realism of the raw output of the decoder (mean of $p_\theta(x|z)$).

A first perspective for further research is to revisit the diverse latent space structures and/or neural encoder and decoder architectures proposed in the literature, and to assess them {\em when coupled with a good identification of the data manifold}. The rationale is that a poor approximation of the data manifold might cause compound failures of the overall VAE, obfuscating the true impact of new latent spaces and neural architectures.

A second research perspective is to design observation models better tailored to the specifics of the dataset, able to identify the local tangent to the data manifold though with a better scalability than based on a full-rank covariance matrix.

\FloatBarrier

\bibliography{bibliography}

\clearpage

\appendix

\section{Detailed analysis on the 2D problem}

\subsection{Impact of learning sharpness $\sigma$}
The reconstruction (Fig. \ref{fig:aux_syn-learned-sigmas}, Left) and generation (Fig. \ref{fig:aux_syn-learned-sigmas}, Right) achieved by VAE when learning sharpness $\sigma$ (LS-VAE) are compared with those of F-VAE (Fig. \ref{fig:aux_syn-sigmas}), where the fixed $\sigma$ varies from 1. to .01. 

\begin{figure}[htpb]
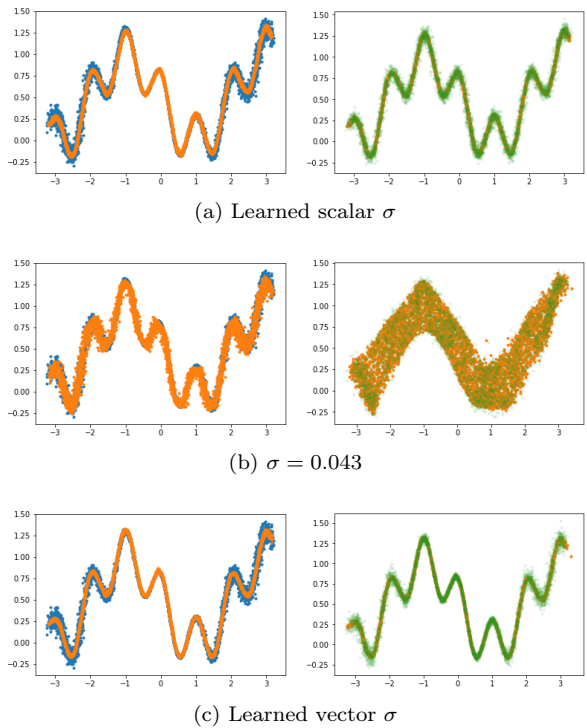

    \centering
    \subfloat[Learned scalar $\sigma$]{
        \includegraphics[width=0.49\columnwidth]{figures/sigma-global.png}
        \includegraphics[width=0.49\columnwidth]{figures/sigma-global-g.png}
    }

    \subfloat[$\sigma = 0.043$]{
        \includegraphics[width=0.49\columnwidth]{figures/sigma-0043.png}
        \includegraphics[width=0.49\columnwidth]{figures/sigma-0043-g.png}
    }

    \subfloat[Learned vector $\sigma$]{
        \includegraphics[width=0.49\columnwidth]{figures/sigma-learned.png}
        \includegraphics[width=0.49\columnwidth]{figures/sigma-learned-g.png}
    }

    \caption{VAE with learned and fixed $\sigma$ (better seen in color). Top: LS-VAE, learned scalar $\sigma$; Medium: F-VAE, fixed $\sigma$ with same value as the eventual $\sigma$ learned by LS-VAE.  Bottom: LV-VAE, learned vector $\sigma$. Left: Reconstructed samples (in orange) and initial samples (in blue). Right: Generated samples (in green) and their average (in orange).}
    \label{fig:aux_syn-learned-sigmas}
\end{figure}

Fig. \ref{fig:aux_syn-sigmas} illustrates the transition of F-VAE as $\sigma$ decreases: A rough approximation of the manifold is obtained for large $\sigma$ values ($\sigma \ge .3$, Fig. \ref{fig:aux_syn-sigmas}.(a-b), Left); A better approximation is obtained for $\sigma = .1$, though F-VAE misses the details and achieves a denoising reconstruction (Fig. \ref{fig:aux_syn-sigmas}.(c), Left); as $\sigma$ decreases, F-VAE learns the data noise as being part of the manifold (Fig. \ref{fig:aux_syn-sigmas}.(d-e), Left). In the meanwhile, the generation is acceptable only for $\sigma = .1$ (Fig. \ref{fig:aux_syn-sigmas}.(c), Right); for higher $\sigma$ values, the generation fails due to a bad approximation of the manifold (for $\sigma > .1$); for smaller $\sigma$ values, the generation failure is due to the fact that the latent space encodes a rich information which is insufficiently compressed.

\begin{figure}[htpb]
    \centering
    \subfloat[$\sigma = 1.0$]{
        \includegraphics[width=0.48\columnwidth]{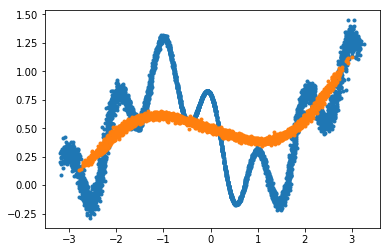}
        \includegraphics[width=0.48\columnwidth]{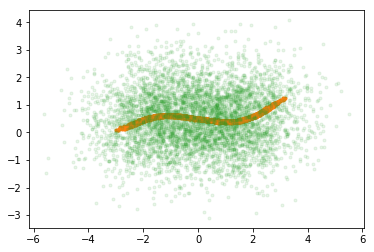}
    }
    
    \subfloat[$\sigma = 0.3$]{
        \includegraphics[width=0.48\columnwidth]{figures/sigma-0300.png}
        \includegraphics[width=0.48\columnwidth]{figures/sigma-0300-g.png}
    }
    
    \subfloat[$\sigma = 0.1$]{
        \includegraphics[width=0.48\columnwidth]{figures/sigma-0100.png}
        \includegraphics[width=0.48\columnwidth]{figures/sigma-0100-g.png}
    }
    
    \subfloat[$\sigma = 0.03$]{
        \includegraphics[width=0.48\columnwidth]{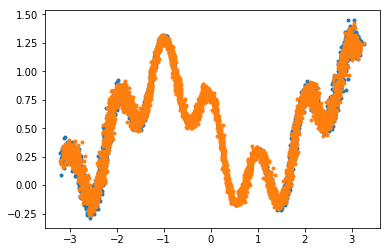}
        \includegraphics[width=0.48\columnwidth]{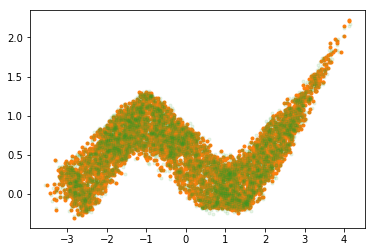}
    }
    
    \subfloat[$\sigma = 0.01$]{
        \includegraphics[width=0.48\columnwidth]{figures/sigma-0010.png}
        \includegraphics[width=0.48\columnwidth]{figures/sigma-0010-g.png}
    }

    \caption{F-VAE: VAE with fixed $\sigma$ ranging from 1 to .01 from top to bottom (better seen in color). Left: Reconstructed samples (in orange) and initial samples (in blue). Right: Generated samples (in green) and their average (in orange).}
    \label{fig:aux_syn-sigmas}
\end{figure}

\FloatBarrier

\subsection{Training dynamics}

Likewise, Figs. \ref{fig:aux_fixed-sigma-loss} and \ref{fig:aux_learned-sigma-loss} display the evolution of the training losses for respectively the fixed $\sigma$ case (F-VAE, Fig. \ref{fig:aux_fixed-sigma-loss}) and the learned one (LS-VAE, Fig. \ref{fig:aux_learned-sigma-loss}). 

In the early learning stage, the reconstruction is bad. For a fixed, small value of $\sigma$ (Fig. \ref{fig:aux_fixed-sigma-loss}), the reconstruction loss is amplified by the $\frac{1}{2\sigma^2}$ factor:

\begin{equation}
    - \log p_\theta(x | z) = \frac{1}{2 \sigma^2}\| x - \mu(z) \|^2 + \log \sigma
\end{equation}

The reconstruction then quickly improves as the model learns to auto-encode the data, and all losses stabilize as the auto-encoding is fine-tuned during the rest of the training period.

\begin{figure}[htb]
    \centering
    \includegraphics[width=\columnwidth]{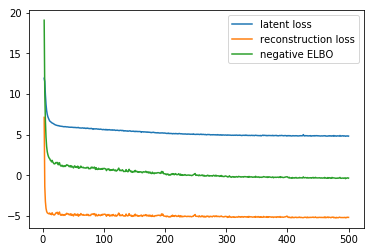}
    \caption{F-VAE: Reconstruction and latent losses along the training process for a fixed $\sigma = 0.043$.}
    \label{fig:aux_fixed-sigma-loss}
\end{figure}

\begin{figure}[htb]
    \centering
        \includegraphics[width=\columnwidth]{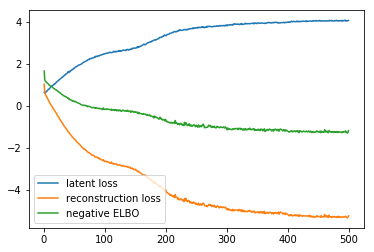}
    \caption{LS-VAE: Reconstruction and latent losses along the training process for a learned scalar $\sigma$. }
    \label{fig:aux_learned-sigma-loss}
\end{figure}

\newpage

On the contrary, when $\sigma$ is learned (Fig. \ref{fig:aux_learned-sigma-loss}), the reconstruction loss is moderate in the early learning stage, as $\sigma$ is initialized to 1. All losses are moderate at this early stage. The reconstruction loss gently decreases along the training period. Its decreases coincide with the decreases of the learned $\sigma$,  illustrated on Fig. \ref{fig:aux_learned-sigma-curve}. Simultaneously, the latent loss also gently increases. Eventually, the reconstruction loss is the same for F-VAE and LS-VAE (circa $-5$); but the latent loss is significantly lower for LS-VAE than for F-VAE ($\approx 4$ compared to $\approx 5$), yielding a better final ELBO, and therefore a higher generation quality. Our tentative interpretation is that the learning of $\sigma$ enables a gradual increase of the information stored in the latent space, allowing the VAE to converge toward a more compressed latent representation.

\begin{figure}[htb]
    \centering
        \includegraphics[width=\columnwidth]{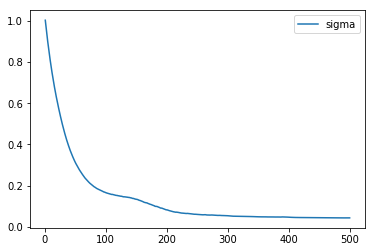}
    \caption{LS-VAE: Evolution of the learned value for $\sigma$. }
    \label{fig:aux_learned-sigma-curve}
\end{figure}

This tentative interpretation is confirmed by complementary experiments, as follows. F-VAE is launched with a fixed $\sigma$, the value of which is set according to the dynamic schedule given on Fig. \ref{fig:aux_learned-sigma-curve}.

For all runs (with different random initialization of the encoder and decoder modules), F-VAE with this dynamic schedule of $\sigma$ consistently yields the same loss curves, final ELBO and observed generation quality as LS-VAE. 
These experiments tend to confirm that the $\sigma$ learning scheme allows LS-VAE to set and follow an efficient annealing optimization scheme. 

\FloatBarrier
\newpage

\section{Complementary experiments on MNIST and CelebA}

This section presents in more detail the results obtained by 
F-VAE and LV-VAE on  MNIST and CelebA. 

Same neural network architectures are used for F-VAE and LV-VAE: the only difference is that LV-VAE considers twice as much output channels as F-VAE (associating to each image pixel a pair $(\mu, \log \sigma)$ instead of $\mu$). 

It is emphasized that the training cost is identical for both F-VAE and LV-VAE.

\subsection{MNIST}

\paragraph{Neural architectures.} The encoder and decoder architectures are symmetrical, with 3 convolution layers followed by 4 residual blocks for the encoder, and 4 residual blocks followed by 3 transposed convolutions for the decoder. The latent space dimension is 256. The training time is 5 hours (1,000 epochs). 

Figs. \ref{fig:aux_mnist-more-se-rec} and \ref{fig:aux_mnist-more-learned-rec} respectively illustrate the reconstruction with F-VAE (square error loss, equivalent to  $\sigma = \frac{1}{\sqrt{2}}$) and LV-VAE (where $\sigma$ is learned as an output of the decoder). The difference is moderate, as MNIST involves high-contrast images with few details. Still,  F-VAE reconstructs the images as a smoothed version of themselves, with smoother edges and imperfections removed. In contrast, LV-VAE more closely matches all the details of the original data.
\def\FVAE{F-VAE (fixed $\sigma = \frac{1}{\sqrt{2}}$)}
\def\LVAE{LV-VAE ($\sigma$ learned as decoder output)}

\begin{figure}[htbp]
    \centering
    \includegraphics[width=\columnwidth]{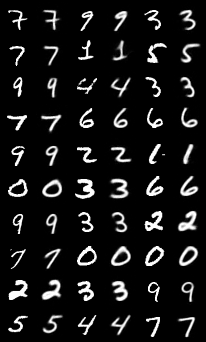}
    \caption{\FVAE: Reconstruction on MNIST.}
    \label{fig:aux_mnist-more-se-rec}
\end{figure}

\begin{figure}[htbp]
    \centering
    \includegraphics[width=\columnwidth]{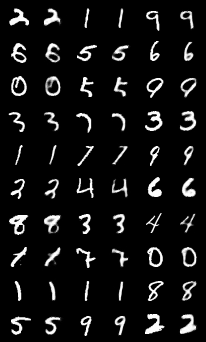}
    \caption{\LVAE: Reconstruction on MNIST.} 
    \label{fig:aux_mnist-more-learned-rec}
\end{figure}

Figs. \ref{fig:aux_mnist-more-se-gen} and \ref{fig:aux_mnist-more-learned-gen} likewise illustrate the generation with F-VAE and LV-VAE, comparing for each model the mean of the output $p_\theta(x|z)$ with actual samples from it.
A striking finding is the difference between the samples and their expectation for F-VAE (Fig. \ref{fig:aux_mnist-more-se-gen}): the former ones are close to random noise while the latter ones are smoothed realistic images. 

As said (section 3.2), this random noise effect explains why the means are preferred to samples, inducing some discrepancy between the VAE design (F-VAE being optimized to sample $p_\theta(x|z)$) and its usage. 

On the contrary, LV-VAE yields precise distributions
(Fig. \ref{fig:aux_mnist-more-learned-gen}), as illustrated by the similarity between the mean of $p_\theta(x|z)$ and its samples. The slight differences (noise on the edge of the shapes) are explained as LV-VAE retains larger $\sigma$ values at the frontier between the black and white regions in the image.

As a counterpart for its ability to reconstruct details, LV-VAE stores more information in the latent space, making it more difficult to compress, and leading to less realistic generated images. The bottleneck here comes from the  insufficient expressivity of the encoder and decoder architectures; a more powerful model is needed to appropriately compress the data.

\begin{figure}[htbp]
    \centering
    \includegraphics[width=\columnwidth]{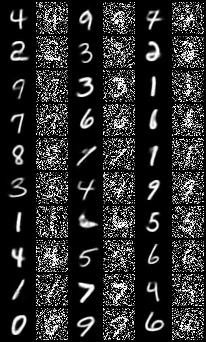}
    \caption{\FVAE: Generations (mean and sample of $p_\theta(x|z)$) on MNIST.}
    \label{fig:aux_mnist-more-se-gen}
\end{figure}

\begin{figure}[htbp]
    \centering
    \includegraphics[width=\columnwidth]{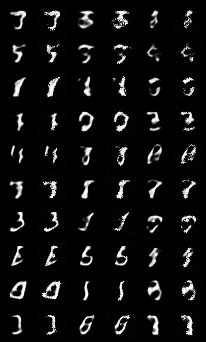}
    \caption{\LVAE: Generations (mean and sample of $p_\theta(x|z)$) on MNIST.}
    \label{fig:aux_mnist-more-learned-gen}
\end{figure}

\FloatBarrier

\subsection{CelebA}

\paragraph{Neural architectures.} The encoder and decoder architectures are similar to those used for MNIST, but with 4 convolutions layers, 5 residual blocks, and a latent space dimension of 2048. The training time is 
ca 24 hours (200 epochs).

The reconstruction and generation performances of F-VAE and LV-VAE on CelebA are  interpreted in much the same way as on MNIST, with the fact that the human eye is more able to identify unrealistic data when it comes to human faces. Figs. \ref{fig:aux_celeba-more-se-rec} and \ref{fig:aux_celeba-more-learned-rec} show the reconstructions from F-VAE and LV-VAE. Like for MNIST, the F-VAE reconstructions (Fig. \ref{fig:aux_celeba-more-se-rec}) are smoothed out, and lacking detail (see e.g. the hair or the tip of the nose). Again, the LV-VAE reconstructions (Fig. \ref{fig:aux_celeba-more-learned-rec}) are more detailed and almost pixel perfect.

\begin{figure}[htb]
    \centering
    \includegraphics[width=\columnwidth]{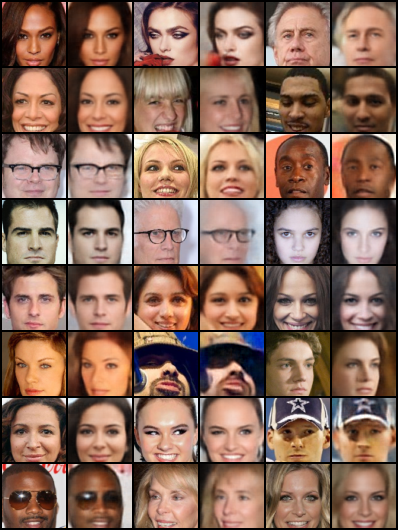}
    \caption{\FVAE: Reconstructions on CelebA.}
    \label{fig:aux_celeba-more-se-rec}
\end{figure}

\begin{figure}[htb]
    \centering
    \includegraphics[width=\columnwidth]{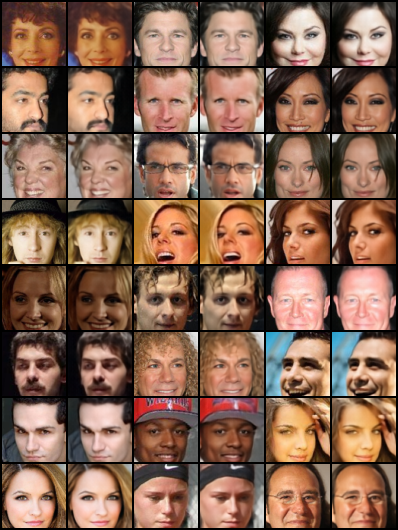}
    \caption{\LVAE: Reconstructions on CelebA.}
    \label{fig:aux_celeba-more-learned-rec}
\end{figure}

\FloatBarrier

Figs. \ref{fig:aux_celeba-more-se-gen} and \ref{fig:aux_celeba-more-learned-gen} likewise illustrate the generation with F-VAE and LV-VAE, showing actual samples of $p_\theta(x|z)$ and their expectation for each model.  The remarks made for MNIST still hold: for F-VAE, the samples  show a very high level of noise, while the means of $p_\theta(x|z)$ have a similar quality as the reconstructions. For LV-VAE, the samples are very close to the means, indicating that the learned variance is low. While the generated images are detailed, these details however lack a global consistency, harming the overall realism of the generated images. 
This limitation is likewise explained from the insufficient expressivity of the encoder and decoder neural networks, as well as the latent space structure.

\begin{figure}[htb]
    \centering
    \includegraphics[width=\columnwidth]{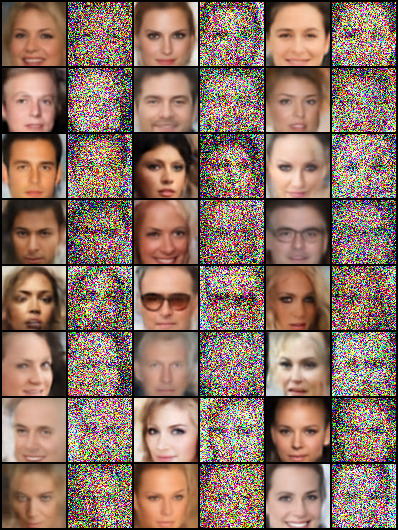}
    \caption{\FVAE: Generations (mean and sample of $p_\theta(x|z)$) on CelebA.}
    \label{fig:aux_celeba-more-se-gen}
\end{figure}

\begin{figure}[htb]
    \centering
    \includegraphics[width=\columnwidth]{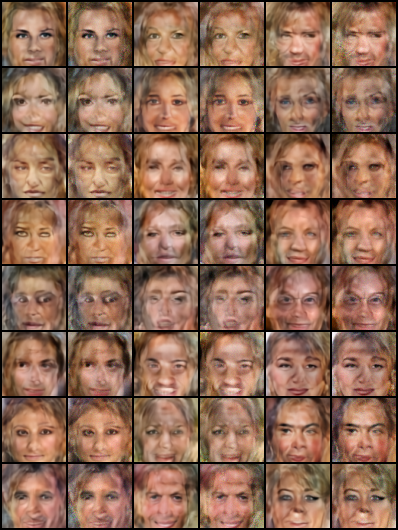}
    \caption{\LVAE: Generations (mean and sample of $p_\theta(x|z)$) on CelebA.}
    \label{fig:aux_celeba-more-learned-gen}
\end{figure}

\end{document}